\documentclass[11pt, twoside]{article}
\usepackage[utf8]{inputenc}
\usepackage[T1]{fontenc}
\usepackage{lmodern}
\usepackage{booktabs}
\usepackage[dvipsnames]{xcolor}
\usepackage{pifont}
\definecolor{darkgreen}{rgb}{0.0, 0.5, 0.0}

\usepackage{amsmath, amssymb, amsthm, mathtools}
\usepackage{bm}
\usepackage{comment}
\usepackage{geometry}
\usepackage[square,numbers]{natbib}
\usepackage{graphicx}
\usepackage{xcolor}
\usepackage{color}
\usepackage{booktabs}

\usepackage{algorithm}
\usepackage{algpseudocode}
\usepackage{float}
\usepackage{placeins}

\usepackage{enumitem}
\setlist[itemize]{leftmargin=*}
\setlist[enumerate]{leftmargin=*}

\usepackage[hidelinks]{hyperref}
\usepackage{url}

\numberwithin{equation}{section}
\newtheorem{theorem}{Theorem}[section]

\newtheorem{remark}[theorem]{Remark}

\newcommand{\Prob}{\mathbb{P}}
\newcommand{\calG}{\mathcal{G}}
\newcommand{\Ghat}{\widehat{\mathcal{G}}}

\newcommand{\abs}[1]{\left\lvert #1 \right\rvert}

\newcommand{\Tin}{T_{\mathrm{in}}}
\newcommand{\Tout}{T_{\mathrm{out}}}

\newcommand{\veps}{\varepsilon}

\title{Operator learning for the 2D incompressible Navier-Stokes equations: a conformal prediction approach in the data-scarce regime}

\author{
Weinan Wang\textsuperscript{1}\quad
Bowen Gang\textsuperscript{2}\quad
Hao Deng\textsuperscript{2}
}

\date{\today}

\begin{document}
\maketitle

\noindent
\textsuperscript{1}Department of Mathematics, University of Oklahoma, Norman, OK, USA\\
\textsuperscript{2}Department of Statistics and Data Science, School of Management, Fudan University, Shanghai, China\\[2mm]
\noindent
Email: \texttt{bgang@fudan.edu.cn}, \texttt{ww@ou.edu}, \texttt{haodeng0617@gmail.com}

\begin{abstract}
In this paper, we propose a perturbation-based conformal prediction framework for uncertainty quantification in operator learning, with a focus on the 2D Navier--Stokes equations. While neural operators provide fast surrogates for expensive PDE solvers, they do not by themselves provide calibrated uncertainty for spatiotemporal field predictions. Our approach wraps a trained Fourier Neural Operator (FNO) with split conformal prediction and constructs the local uncertainty scale by comparing the predictions of two operators trained on nearly identical datasets: one on the original labels and one on labels perturbed by small Gaussian noise.  We consider this procedure in the data-scarce regime, where the total label budget is fixed and methods that require a separate uncertainty network must divide training data between multiple models.  On the 2D Navier--Stokes benchmark, the perturbation-based method produces substantially narrower conformal bands than 
existing methods under matched total data budgets while maintaining the target simultaneous coverage. These results suggest that perturbation sensitivity is a practical and sample-efficient uncertainty proxy for conformalized neural operators.

\end{abstract}

\tableofcontents

\section{Introduction}\label{sec:intro}

In this paper, we consider operator learning for the 2D incompressible Navier-Stokes equations with a perturbation-based uncertainty scale. Neural operators provide fast surrogates for time-dependent PDE solution maps,
but their point predictions do not quantify the spatially and temporally varying
error of the learned surrogate.  This limitation is especially important for
Navier--Stokes rollouts, where a small number of training trajectories must often
support predictions over a high-dimensional spatiotemporal grid. 
Classical numerical PDE solvers, including finite element, finite volume, and spectral methods, are accurate and well established, yet their computational cost becomes prohibitive in regimes requiring fine spatial resolution, long temporal horizons, or repeated evaluations.
 \citep{quarteroni1994numerical,leveque2002finite,trefethen2000spectral}. For incompressible Navier--Stokes, classical well-posedness and numerical stability considerations motivate grid-based solvers that are accurate but computationally intensive at fine resolution \citep{temam2024navier}. Operator learning offers a complementary, data-driven route: rather than solving a PDE from scratch for each instance, one learns a mapping from inputs to outputs \citep{kovachki2024operator,subedi2025operator,boulle2024guide,nelsen2024randomfeatures}. Neural operators formalize learning resolution-invariant maps between function spaces, making them attractive surrogates for parametric PDE solution operators \citep{kovachki2023neural,winovich2025active}. Neural operator architectures such as the Fourier Neural Operator \citep{li2020fourier} and DeepONet \citep{lu2021learning} are designed to approximate mappings between discretized function spaces and can generalize across discretizations \citep{subedi2025controlling,millard2025split,kovachki2023neural}. This operator-learning viewpoint is also closely related to recent work on inverse problems, including neural inverse operators, probabilistic perspectives on operator learning for inverse problems, stochastic inverse problems, and PDE-constrained inverse-problem formulations \citep{molinaro2023neural,nelsen2025inverse,li2024stochastic,vanleeuwen2025constraint}. In addition, the effectiveness of learned operators depends strongly on the amount and distribution of training data, a point that has recently been emphasized in training-distribution selection and out-of-distribution analysis \citep{guerra2025learning}. However, because neural operators are approximate models trained on finite data, their errors vary across regimes, forecast horizons, and input distributions. In safety-critical or decision-making settings, point predictions are therefore insufficient, and one needs uncertainty quantification (UQ) for \emph{structured outputs}, namely complete spatiotemporal fields \citep{majda2002vorticity,guo2017calibration,ovadia2019can,gopakumar2025calibrated,magnaniKERH25,lakshminarayanan2017simple,bulte2025probabilistic}.

Conformal prediction (CP) provides distribution-free, finite-sample marginal coverage guarantees under exchangeability \citep{vovk2005algorithmic,lei2018distribution,angelopoulos2021gentle}, and recent work has adapted the framework to operator learning for functional outputs \citep{ma2024calibrated,moya2025conformalized,bulte2025probabilistic,yu2025conformal}. The standard validity guarantee is strictly marginal. This marginal guarantee does not ensure uniform performance across specific inputs or individual output coordinates. In practice, we want adaptive prediction bands that maintain approximately uniform conditional coverage at each spatiotemporal point, expanding in regions of high predictive uncertainty and contracting where the surrogate is precise. Achieving exact distribution-free conditional coverage is generally impossible for nontrivial prediction sets \citep{foygel2021limits}. This limitation motivates the pursuit of simultaneous coverage over all output coordinates, combined with a carefully designed nonconformity score to approximate conditional calibration. The adaptivity of the resulting bands depends critically on this score. A well-scaled nonconformity function can produce prediction sets that satisfy marginal or simultaneous validity while still responding to local difficulty. Naive residual-based scores, which normalize absolute residuals by a global or per-coordinate standard deviation, typically yield constant-width bands \citep{papadopoulos2002inductive,lei2014distribution}. Although statistically valid, constant-width bands are inefficient and physically uninformative for fluid dynamics. Reliable uncertainty quantification requires bands that widen in turbulent, high-gradient regions and contract in smoother laminar areas. Existing adaptive methods frequently address this by training auxiliary error models, such as quantile regression networks, to estimate local uncertainty scales \citep{ma2024calibrated,romano2019conformalized}. This approach introduces additional optimization complexity and increases the risk of overfitting the uncertainty model to the calibration set.

 The central idea of this present paper is to measure how much the learned operator changes when the training labels are slightly corrupted: if a particular output coordinate varies substantially under this perturbation, then the corresponding prediction should receive a wider uncertainty budget. Concretely, we train two Fourier Neural Operators with identical architectures, one on the original training data and one on perturbed labels, and use their pointwise disagreement as the raw uncertainty measure. We then conformalize this measure with a max-type nonconformity score to obtain simultaneous coverage over the full spatiotemporal grid. Our experiments focus on the data-scarce regime, where this construction has a natural advantage over methods such as UQNO that must reserve a disjoint subset of labels to train a secondary uncertainty network. Across the Navier--Stokes benchmark, the resulting conformal bands are consistently narrower than those produced by MC Dropout, Laplace approximation, and UQNO under matched total data budgets.
\subsection{Paper organization}
The rest of the paper is organized as follows. Section~\ref{sec:problem} formulates the finite-grid prediction problem and Section~\ref{sec:conformal_background} reviews
the split-conformal guarantee.  Sections~\ref{sec:method}--\ref{sec:expt} present the perturbation-scaled
score, implementation details, and numerical results for the Navier--Stokes
benchmark.

\section{Problem Formulation}\label{sec:problem}
We consider the two-dimensional incompressible Navier--Stokes equations on the periodic domain
$\mathbb{T}^2$ in vorticity form
\begin{equation}\label{eq:NS}
	\partial_t \omega + v\cdot \nabla \omega =  \nu\Delta \omega + f(x),
	\qquad  \ t\in (0,T],
\end{equation}
together with incompressibility $\nabla\cdot v=0$ and initial condition $\omega(x,0)=w_0(x)$.
Here $\omega=\omega(x,t)=\nabla \times v$ denotes vorticity, $v=v(x,t)$ denotes velocity vector field, $\nu>0$ is the viscosity, and $f$ is an external forcing. Under periodic boundary conditions, $v$ can be recovered from $\omega$ through a stream function $\psi$ satisfying $-\Delta \psi = \omega$ and $v=\nabla^\perp \psi$. In the sequel, $u$ and $\mathbf{u}$ denote the future vorticity trajectory and its discretization, not the velocity field.

For a fixed forcing $f$, the PDE induces a solution map that propagates the state forward in time.
In many operator-learning benchmarks, the goal is not to predict a single scalar but to learn the \emph{operator} that maps an input specification (e.g.\ an initial condition) to a future field trajectory.
To reflect the data used in our experiments, we formulate the solution operator on a finite input horizon $\Tin$ and an output horizon $\Tout$ as follows.
Let $\{t_\ell\}_{\ell=1}^{\Tin+\Tout}$ be discrete times with uniform step size $\Delta t$.
Define the continuous-time \textit{input history} and \textit{future rollout} as the collections
\[
a := \{w(\cdot,t_1),\ldots,w(\cdot,t_{\Tin})\},
\qquad
u := \{w(\cdot,t_{\Tin+1}),\ldots,w(\cdot,t_{\Tin+\Tout})\}.
\]
We view $a$ and $u$ as elements of suitable function spaces (e.g.\ $L^2(\mathbb{T}^2)$ at each time slice).
The corresponding (unknown) solution operator is
\begin{equation}\label{eq:operator_mapping}
	\calG:\ a \mapsto u.
\end{equation}
The unknown operator $\calG$ depends implicitly on $\nu$, $f$, and on the time grid. In dataset generation, randomness enters through the sampling of initial conditions, inducing a distribution over $(a,u)$ pairs.
Ideally, we would like to construct an uncertainty set $\mathcal{C}$ for the operator $\calG$ such that 
\begin{equation}\label{eq:ideal_goal}
		\Prob\big(\mathcal{G} \in \mathcal{C}\big) \geq 1 - \alpha
\end{equation}
for a prescribed confidence level $1-\alpha$. In practice, however, we work with discretized representations of these continuous objects.

\subsection{Discretization and Uncertainty Quantification}
\medskip

We use non-bold symbols such as $a,u$ for the continuous fields or trajectories and bold symbols such as $\mathbf a,\mathbf u$ for their vectorized discretizations. The observed data are
\[
\mathcal D_{\rm obs}=\{(\mathbf a_i,\mathbf u_i)\}_{i=1}^M.
\]
The subscript indexes the sample, while superscripts index vector entries; for example, $\mathbf a_i^k$ denotes the $k$th component of $\mathbf a_i$.

The operator learning problem is defined on a finite spatial grid $\mathcal{D}_h = \{x_1, \dots, x_N\} \subset \mathbb{T}^2$ (e.g., $N=n_x n_y$ for a uniform mesh) and a discrete output timeline $\mathcal{T}_{\mathrm{out}} = \{t_{\Tin+1}, \dots, t_{\Tin+\Tout}\}$.
We define the finite-dimensional input vector $\mathbf{a} \in \mathbb{R}^{d_{\mathrm{in}}}$ (where $d_{\mathrm{in}} = N \Tin$) by flattening the vorticity field values over the grid $\mathcal{D}_h$ at the $\Tin$ input time steps. Similarly, we define the target output vector $\mathbf{u} \in \mathbb{R}^{d_{\mathrm{out}}}$ (where $d_{\mathrm{out}} = N \Tout$) by stacking the field values over $\mathcal{D}_h$ at the $\Tout$ output time steps.
With this convention, the discrete ground-truth operator $\mathcal{G}$ acts as a mapping between finite-dimensional Euclidean spaces:
\begin{equation}\label{eq:ghat_map}
	\mathcal{G}: \mathbb{R}^{d_{\mathrm{in}}} \to \mathbb{R}^{d_{\mathrm{out}}}.
\end{equation}
Given a new query input $\mathbf{a}_{\mathrm{new}}$, our objective is to construct an uncertainty set $\mathcal{C}(\mathbf{a}_{\mathrm{new}}) \subseteq \mathbb{R}^{d_{\mathrm{out}}}$ that satisfies the following marginal coverage guarantee:
\begin{equation}\label{eq:goal}
	\mathbb{P}\left( \mathbf{u}_{\mathrm{new}} \in \mathcal{C}(\mathbf{a}_{\mathrm{new}}) \right) \geq 1 - \alpha,
\end{equation}
for a prespecified miscoverage rate $\alpha \in (0,1)$. We assume that the new data point $(\mathbf{a}_{\mathrm{new}}, \mathbf{u}_{\mathrm{new}})$ and the observed data $\mathcal{D}_{\text{obs}}$ are exchangeable \citep{tibshirani2019conformal,barber2023conformal}.

\begin{remark}[Pointwise vs. Simultaneous Coverage]
	It is important to distinguish the goal in \eqref{eq:goal} from that in \eqref{eq:ideal_goal}. Our formulation constructs a prediction set for the output of $\mathcal{G}$ at a \emph{specific, random} test point. A rigorous uncertainty set for the operator $\mathcal{G}$ itself would require a \emph{simultaneous coverage} guarantee. That is, one would seek a set-valued mapping $\mathcal{C}: \mathbb{R}^{d_{\mathrm{in}}} \to 2^{\mathbb{R}^{d_{\mathrm{out}}}}$ satisfying:
	\begin{equation}\label{eq:simultaneous_goal}
		\mathbb{P}_{\mathcal{D}_{\text{obs}}}\left( \mathcal{G}(\mathbf{a}) \in \mathcal{C}(\mathbf{a}), \quad \forall \mathbf{a} \in \mathbb{R}^{d_{\mathrm{in}}} \right) \geq 1 - \alpha.
	\end{equation}
	The condition in \eqref{eq:simultaneous_goal} requires $\mathcal{C}(\mathbf{a})$ to contain the ground-truth trajectory $\mathcal{G}(\mathbf{a})$ for {every} possible input $\mathbf{a}$ in the domain simultaneously. While theoretically significant, achieving such uniform convergence guarantees requires strict assumptions on the complexity of $\mathcal{G}$ and is computationally demanding; we therefore focus on achieving \eqref{eq:goal}.
	\end{remark}

\section{Preliminaries: Split Conformal Prediction}\label{sec:conformal_background}

We briefly review the split conformal prediction framework, a distribution-free method for constructing valid predictive uncertainty sets. Let $\mathcal{D}_n = \{(X_i, Y_i)\}_{i=1}^n$ be a dataset of $n$ exchangeable input-output pairs drawn from a joint distribution $P_{XY}$ over $\mathcal{X} \times \mathcal{Y}$. For a new test point $X_{n+1}$, our goal is to construct a prediction set $\mathcal{C}(X_{n+1}) \subseteq \mathcal{Y}$ that contains the unknown response $Y_{n+1}$ with probability at least $1-\alpha$.

In the split conformal protocol, the dataset $\mathcal{D}_n$ is randomly partitioned into two disjoint subsets: a \emph{proper training set} indexed by $\mathcal{I}_{\text{train}}$ and a \emph{calibration set} indexed by $\mathcal{I}_{\text{cal}}$, such that $\mathcal{I}_{\text{train}} \cup \mathcal{I}_{\text{cal}} = \{1, \dots, n\}$ \citep{papadopoulos2002inductive}.
The training data \(\{(X_i,Y_i):i\in\mathcal{I}_{\mathrm{train}}\}\) are used only to learn a nonconformity score function
\[
S:\ \mathcal{X}\times\mathcal{Y}\to\mathbb{R},
\]
which quantifies how atypical a candidate pair \((x,y)\) is. Once $S$ is fitted on the training set it is treated as fixed during calibration. Nonconformity scores are then computed on the calibration data:
\[
V_i \;:=\; S(X_i,Y_i),\qquad i\in\mathcal{I}_{\mathrm{cal}}.
\]
Write \(\{V_{(1)}\le V_{(2)}\le\cdots\le V_{(n_{\mathrm{cal}})}\}\) for the ordered calibration scores and set
\[
k \;:=\; \big\lceil (|\mathcal{I}_{\mathrm{cal}}|+1)(1-\alpha)\big\rceil,\qquad
\widehat{Q}_{1-\alpha}\;:=\; V_{(k)}.
\]
When $k>n_{\mathrm{cal}}$, the usual split-conformal convention is to set $\widehat{Q}_{1-\alpha}=+\infty$; equivalently, finite nontrivial bands at level $1-\alpha$ require $\lceil (n_{\mathrm{cal}}+1)(1-\alpha)\rceil\le n_{\mathrm{cal}}$. 
The split conformal prediction set for a new input \(x\) is then
\begin{equation}\label{eq:cp_set}
	\mathcal{C}(x)
	\;=\;
	\{\,y\in\mathcal{Y}:\ S(x,y)\le \widehat{Q}_{1-\alpha}\,\}.
\end{equation}
We have the following theoretical guarantee. The proof can be found in \citet{angelopoulos2021gentle}.
\begin{theorem}[Marginal Coverage]\label{thm:cp_coverage}
	Assume that the data points $\{(X_i, Y_i)\}_{i=1}^{n+1}$ are exchangeable. Let $\mathcal{C}(X_{n+1})$ be the set constructed via split conformal prediction as defined in \eqref{eq:cp_set}. Then, for any distribution $P_{XY}$ and any choice of score function $S$,
	\begin{equation}
		\mathbb{P}\left( Y_{n+1} \in \mathcal{C}(X_{n+1}) \right) \geq 1 - \alpha.
	\end{equation}
\end{theorem}

%



\section{Methodology: Perturbation-Based Nonconformity Score}\label{sec:method}

The marginal coverage guarantee of split conformal prediction holds regardless of the choice of the nonconformity score function, provided the data are exchangeable. However, the efficiency of the resulting prediction bands depends critically on the uncertainty scale used to normalize the residuals. A naive choice such as the unnormalized residual $S(\mathbf{a}, \mathbf{u}) = \|\mathbf{u} - \widehat{\mathcal{G}}(\mathbf{a})\|_\infty$ produces constant-width bands, which are typically too conservative for operator-valued predictions on heterogeneous spatiotemporal domains.

To construct adaptive uncertainty sets, we employ a normalized max-type nonconformity score of the form
\begin{equation}\label{eq:normalized_score_template}
	S(\mathbf{a}, \mathbf{u}) =
	\max_{k\le d_{\mathrm{out}}}
	\frac{\abs{\mathbf{u}^{k} - \widehat{\mathcal{G}}(\mathbf{a})^{k}}}
	{\sigma(\mathbf{a})^{k}},
\end{equation}
where $k$ indexes the spatiotemporal output coordinates. The efficiency of the resulting conformal bands depends critically on the local scale $\sigma(\mathbf{a})^{k}$, which should reflect the predictive uncertainty at each output coordinate. An effective scale must adapt to spatially varying difficulty: it should assign wider prediction bands to coordinates where the surrogate is sensitive to training noise, and narrower bands where the operator response remains stable. Our construction is guided by a stability principle: if a particular output coordinate varies substantially under small perturbations of the training labels, the corresponding prediction should be treated as uncertain; conversely, stable coordinates should receive a smaller scale. We operationalize this principle by measuring the pointwise disagreement between two neural operators trained on nearly identical datasets. The explicit construction of $\sigma(\mathbf{a})$ from this disagreement, together with the necessary spatial smoothing and floor stabilization, is detailed in the following subsection.

\subsection{Perturbation-Based Uncertainty Scale and Calibration}\label{sec:method:construction}
Let $\{(\mathbf{a}_i,\mathbf{u}_i)\}_{i\in\mathcal{I}_{\mathrm{train}}}$ denote the training set. We first train a base neural operator $\widehat{\mathcal{G}}$ on the original pairs $(\mathbf{a}_i,\mathbf{u}_i)$. We then construct perturbed labels by adding independent Gaussian noise to the outputs:
\begin{equation}\label{eq:perturbed_labels}
	\widetilde{\mathbf{u}}_i = \mathbf{u}_i + \boldsymbol{\veps}_i,
	\qquad
	\veps_i^{k}\overset{\mathrm{iid}}{\sim}\mathcal{N}(0,\sigma_{\veps}^{2}),
	\qquad
	\sigma_{\veps} = c \times \mathrm{std}(\mathbf{u}_{\mathrm{train}}),
\end{equation}
with default $c=0.05$ and std denotes the standard deviation. A second neural operator, denoted $\widehat{\mathcal{G}}_{\veps}$, is trained on the perturbed pairs $(\mathbf{a}_i,\widetilde{\mathbf{u}}_i)$ using the same architecture and optimization protocol as the base model. Both operators use the full proper training set, so no disjoint subset must be reserved to train a separate model to measure uncertainty. For a new input $\mathbf{a}$, we define the raw perturbation disagreement at output coordinate $k$ by
\begin{equation}\label{eq:raw_disagreement}
	d(\mathbf{a})^{k}
	:=
	\abs{\widehat{\mathcal{G}}(\mathbf{a})^{k}
		-
		\widehat{\mathcal{G}}_{\veps}(\mathbf{a})^{k}},
	\qquad k = 1,\dots,d_{\mathrm{out}}.
\end{equation}
This quantity measures the sensitivity of the learned prediction to small perturbations of the training labels. We intend to use this disagreement as a local scale in a normalized nonconformity score, which divides the absolute residual at each coordinate by the corresponding scale. Because the scale appears in the denominator and the final score is defined as a maximum over all output coordinates, isolated near-zero values of $d(\mathbf{a})^{k}$ would produce arbitrarily large ratios for individual calibration samples.  To prevent this, we apply two regularization steps.

First, we smooth the disagreement map spatially at each output time step. Writing a coordinate as $k \equiv (s_{1}, s_{2}, t)$, we define
\begin{equation}\label{eq:smoothed_disagreement}
	\bar d(\mathbf{a})^{(s_{1}, s_{2}, t)}
	:=
	\frac{1}{|\mathcal{N}_{K}(s_{1}, s_{2})|}
	\sum_{(s_{1}', s_{2}') \in \mathcal{N}_{K}(s_{1}, s_{2})}
	d(\mathbf{a})^{(s_{1}', s_{2}', t)},
\end{equation}
where $\mathcal{N}_{K}(s_{1}, s_{2})$ is a $K \times K$ spatial neighborhood centered at $(s_{1}, s_{2})$ with appropriate boundary handling. In the experiments, we use a $15 \times 15$ smoothing window. Second, to ensure the denominator remains uniformly bounded away from zero, we impose a data-dependent floor computed from the proper training set:
\begin{equation}\label{eq:sigma_def}
	\sigma(\mathbf{a})^{k}
	:=
	\max\!\bigl\{\bar d(\mathbf{a})^{k},\,\tau_{0}\bigr\},
	\qquad
	\tau_{0}
	:=
	0.1 \times
	\mathrm{median}\Bigl\{
	d(\mathbf{a}_{i})^{k}
	:
	i\in\mathcal{I}_{\mathrm{train}},\;
	k = 1,\dots,d_{\mathrm{out}}
	\Bigr\}.
\end{equation}

\begin{remark}[Stability Interpretation]
	The perturbation $\boldsymbol{\veps}_{i}$ may be interpreted as a proxy for numerical error in the solver-generated training outputs, or more generally as a small perturbation of the response variable \citep{bishop1995training}. Under this interpretation, the scale $\sigma(\mathbf{a})^{k}$ measures how strongly output coordinate $k$ changes when the training labels are slightly corrupted. Large values of $\sigma(\mathbf{a})^{k}$ indicate locally unstable or hard-to-learn predictions, whereas small values indicate stable predictions that are comparatively insensitive to perturbations in the training data.
\end{remark}

\paragraph{Perturbation-Scaled Nonconformity Score.}
With the stabilized uncertainty map $\sigma(\mathbf{a})$ defined above, we construct the conformal score by normalizing the calibration residuals. For each calibration sample $i \in \mathcal{I}_{\text{cal}}$, let $\mathbf{r}_{i} = \mathbf{u}_{i} - \widehat{\mathcal{G}}(\mathbf{a}_{i})$. We then define
\begin{equation}\label{eq:jacobian_score}
	S_{i}
	=
	\max_{k \in \{1,\dots,d_{\mathrm{out}}\}}
	\frac{\abs{\mathbf{r}_{i}^{k}}}{\sigma(\mathbf{a}_{i})^{k}}.
\end{equation}
By taking the maximum over all output coordinates, this score targets simultaneous coverage over the entire spatiotemporal grid.

\paragraph{Calibration and Prediction Sets.}
Following split conformal prediction, let $S_{(1)} \le \cdots \le S_{(n)}$ denote the ordered calibration scores, where $n = |\mathcal{I}_{\mathrm{cal}}|$. For target coverage level $1 - \alpha$, we define

\begin{equation}\label{eq:quantile_def}
    \widehat{Q}_{1-\alpha} := S_{(\lceil (n+1)(1-\alpha)\rceil)}.
\end{equation}
For a new test input $\mathbf{a}_{\mathrm{new}}$, we compute the base prediction $\widehat{\mathbf{u}}_{\mathrm{new}} = \widehat{\mathcal{G}}(\mathbf{a}_{\mathrm{new}})$ together with the perturbation-derived scale $\sigma(\mathbf{a}_{\mathrm{new}})$, and define the prediction bands coordinate-wise by
\begin{equation}\label{eq:final_bands}
    \mathcal{C}(\mathbf{a}_{\mathrm{new}})^{k}
    =
    \left[
    \widehat{\mathbf{u}}_{\mathrm{new}}^{k}
    -
    \widehat{Q}_{1-\alpha}\sigma(\mathbf{a}_{\mathrm{new}})^{k},
    \quad
    \widehat{\mathbf{u}}_{\mathrm{new}}^{k}
    +
    \widehat{Q}_{1-\alpha}\sigma(\mathbf{a}_{\mathrm{new}})^{k}
    \right].
\end{equation}
We have the following theoretical guarantee.

\begin{theorem}\label{thm:simultaneous_coverage}
    Assume the calibration data $\{(\mathbf{a}_i,\mathbf{u}_i)\}_{i \in \mathcal{I}_{\mathrm{cal}}}$ and the test point $(\mathbf{a}_{\mathrm{new}},\mathbf{u}_{\mathrm{new}})$ are exchangeable. Let $\mathcal{C}(\mathbf{a}_{\mathrm{new}})$ be defined by \eqref{eq:final_bands}. Then
    \begin{equation}
        \Prob\!\left(
        \mathbf{u}_{\mathrm{new}}^{k} \in
        \mathcal{C}(\mathbf{a}_{\mathrm{new}})^{k},
        \quad
        \forall k \in \{1,\dots,d_{\mathrm{out}}\}
        \right)
        \ge 1 - \alpha.
    \end{equation}
\end{theorem}

\begin{proof}
	Let $\mathcal{D}_{\mathrm{train}} = \{(\mathbf{a}_i, \mathbf{u}_i)\}_{i \in \mathcal{I}_{\mathrm{train}}}$ denote the training set, and let $\boldsymbol{\varepsilon}$ denote the collection of Gaussian perturbations used to construct the noisy labels for the second operator $\widehat{\mathcal{G}}_{\veps}$. Conditional on the $\sigma$-algebra generated by $\mathcal{D}_{\mathrm{train}}$ and $\boldsymbol{\varepsilon}$, both the local scale mapping $\mathbf{a} \mapsto \sigma(\mathbf{a})$ defined in~\eqref{eq:sigma_def} and the nonconformity score function
	\[
	S(\mathbf{a}, \mathbf{u}) = \max_{k \in \{1,\dots,d_{\mathrm{out}}\}} \frac{|\mathbf{u}^k - \widehat{\mathcal{G}}(\mathbf{a})^k|}{\sigma(\mathbf{a})^k}
	\]
	are deterministic. In particular, they do not depend on the calibration or test data.
	
	Define the calibration scores $S_i = S(\mathbf{a}_i, \mathbf{u}_i)$ for $i \in \mathcal{I}_{\mathrm{cal}}$ and the test score $S_{\mathrm{new}} = S(\mathbf{a}_{\mathrm{new}}, \mathbf{u}_{\mathrm{new}})$. Because the augmented collection $\{(\mathbf{a}_i, \mathbf{u}_i)\}_{i \in \mathcal{I}_{\mathrm{cal}} \cup \{\mathrm{new}\}}$ is exchangeable and the score function $S$ is identical and fixed for each pair, the resulting scores $\{S_i\}_{i \in \mathcal{I}_{\mathrm{cal}}} \cup \{S_{\mathrm{new}}\}$ are exchangeable. Let $n = |\mathcal{I}_{\mathrm{cal}}|$, and let $S_{(1)} \le \cdots \le S_{(n)}$ denote the ordered calibration scores. The split-conformal quantile is constructed as $\widehat{Q}_{1-\alpha} = S_{(\lceil (n+1)(1-\alpha) \rceil)}$ as in~\eqref{eq:quantile_def}. By Theorem~1 of \citet{angelopoulos2021gentle} we have
	\[
	\Prob\!\left( S_{\mathrm{new}} \le \widehat{Q}_{1-\alpha} \,\big|\, \mathcal{D}_{\mathrm{train}}, \boldsymbol{\varepsilon} \right) \ge 1 - \alpha.
	\]
	By the definition of the max-type score, the event $S_{\mathrm{new}} \le \widehat{Q}_{1-\alpha}$ is equivalent to the system of inequalities
	\[
	\frac{|\mathbf{u}_{\mathrm{new}}^k - \widehat{\mathcal{G}}(\mathbf{a}_{\mathrm{new}})^k|}{\sigma(\mathbf{a}_{\mathrm{new}})^k} \le \widehat{Q}_{1-\alpha}, \qquad k = 1, \dots, d_{\mathrm{out}}.
	\]
	Rearranging each inequality yields
	\[
	\mathbf{u}_{\mathrm{new}}^k \in \left[ \widehat{\mathcal{G}}(\mathbf{a}_{\mathrm{new}})^k - \widehat{Q}_{1-\alpha}\,\sigma(\mathbf{a}_{\mathrm{new}})^k, \; \widehat{\mathcal{G}}(\mathbf{a}_{\mathrm{new}})^k + \widehat{Q}_{1-\alpha}\,\sigma(\mathbf{a}_{\mathrm{new}})^k \right], \qquad \forall k.
	\]
	This is precisely the simultaneous containment event $\mathbf{u}_{\mathrm{new}}^k \in \mathcal{C}(\mathbf{a}_{\mathrm{new}})^k$ for all $k$, where $\mathcal{C}(\mathbf{a}_{\mathrm{new}})$ is defined in~\eqref{eq:final_bands}. Integrating over the conditioning variables $\mathcal{D}_{\mathrm{train}}$ and $\boldsymbol{\varepsilon}$ via the law of total probability preserves the inequality, giving the unconditional guarantee
	\[
	\Prob\!\left( \mathbf{u}_{\mathrm{new}}^k \in \mathcal{C}(\mathbf{a}_{\mathrm{new}})^k, \; \forall k \in \{1,\dots,d_{\mathrm{out}}\} \right) \ge 1 - \alpha.
	\]
\end{proof}

Theorem \ref{thm:simultaneous_coverage} ensures that, with probability at least $1-\alpha$, the entire predicted trajectory is simultaneously covered on the discrete spatiotemporal evaluation grid. We summarize the resulting procedure in Algorithm~\ref{alg:cp_rewrite}.

\begin{algorithm}[t]
    \caption{Perturbation-based split conformal prediction for neural operators}\label{alg:cp_rewrite}
    \begin{algorithmic}[1]
        \Require Proper training set $\{(\mathbf{a}_i,\mathbf{u}_i)\}_{i\in\mathcal{I}_{\mathrm{train}}}$, calibration set $\{(\mathbf{a}_i,\mathbf{u}_i)\}_{i\in\mathcal{I}_{\mathrm{cal}}}$, miscoverage level $\alpha$, smoothing window $K$
        \State Train base operator $\Ghat$ on $\{(\mathbf{a}_i,\mathbf{u}_i)\}_{i\in\mathcal{I}_{\mathrm{train}}}$
        \State Form perturbed labels $\widetilde{\mathbf{u}}_i = \mathbf{u}_i + \boldsymbol{\veps}_i$ via \eqref{eq:perturbed_labels}
        \State Train perturbed operator $\widehat{\mathcal{G}}_{\veps}$ on $\{(\mathbf{a}_i,\widetilde{\mathbf{u}}_i)\}_{i\in\mathcal{I}_{\mathrm{train}}}$
        \State Compute the floor $\tau_0$ via \eqref{eq:sigma_def} on the proper training set
        \For{$i \in \mathcal{I}_{\mathrm{cal}}$}
        \State $\hat{\mathbf{u}}_i \gets \Ghat(\mathbf{a}_i)$; $\mathbf{r}_i \gets \mathbf{u}_i - \hat{\mathbf{u}}_i$
        \State Compute $d(\mathbf{a}_i)$ via \eqref{eq:raw_disagreement}
        \State Smooth $d(\mathbf{a}_i)$ spatially to obtain $\bar d(\mathbf{a}_i)$ via \eqref{eq:smoothed_disagreement}
        \State Compute $\sigma(\mathbf{a}_i)$ via \eqref{eq:sigma_def}
        \State Compute $S_i$ via \eqref{eq:jacobian_score}
        \EndFor
        \State $\widehat{Q}_{1-\alpha} \gets S_{(\lceil (n+1)(1-\alpha)\rceil)}$
        \State \textbf{Inference:} for a new input $\mathbf{a}_{\mathrm{new}}$, compute $\hat{\mathbf{u}}_{\mathrm{new}} = \Ghat(\mathbf{a}_{\mathrm{new}})$, compute $\sigma(\mathbf{a}_{\mathrm{new}})$ by the same perturbation/smoothing/floor procedure, and output bands \eqref{eq:final_bands}
\end{algorithmic}
\end{algorithm}
\subsection{Practical Implementation}
In the experiments, we parameterize both $\widehat{\mathcal{G}}$ and $\widehat{\mathcal{G}}_{\veps}$ with the same Fourier Neural Operator architecture \citep{li2020fourier}. The input is a stack of $\Tin=10$ vorticity snapshots together with the two spatial coordinates, and the output is a stack of $\Tout=10$ future vorticity fields on the same $64\times 64$ grid. In our implementation, the lifting layer maps $\mathbb{R}^{\Tin+2}$ to $\mathbb{R}^{32}$, the operator contains four spectral convolution blocks retaining the first $12$ Fourier modes in each spatial dimension, and the projection head is given by
\begin{equation}\label{eq:fno_layer}
  \mathbb{R}^{32}
  \longrightarrow
  \mathbb{R}^{128}
  \longrightarrow
  \mathbb{R}^{\Tout},
\end{equation}
with GELU \citep{hendrycks2016gaussian} activations and dropout rate $0.1$ in the base model. Both operators are trained with Adam\citep{kingma2015adam}, StepLR decay, batch size $10$, and relative $L^{2}$ loss. 



\section{Experiments}\label{sec:expt}
In this section, we evaluate the perturbation-based conformal method on the standard 2D incompressible Navier--Stokes vorticity benchmark.
The experimental emphasis is the \emph{data-scarce regime}: all methods are given the same total label budget, and methods that require an additional uncertainty model must divide this limited budget across several roles.
An anonymized code package for reproducing the experiment results will be made available as supplementary material.
\subsection{Dataset and experimental protocol}\label{sec:dataset_protocol}
The dataset consists of $1200$ Navier--Stokes vorticity trajectories with viscosity $\nu=10^{-5}$ on a $64\times64$ periodic spatial grid. Each sample contains $\Tin=10$ input vorticity frames and $\Tout=10$ future output frames, so each prediction has
\[
    d_{\mathrm{out}}=64\times64\times10=40{,}960
\]
output coordinates. We compare the following five uncertainty quantification procedures.
\begin{itemize}
    \item \textbf{MC Dropout.} Following \citet{gal2016dropout}, dropout remains active at inference. We estimate the predictive mean and variance from $20$ independent stochastic forward passes and use the resulting standard deviation as the local scale in the conformal score.
    \item \textbf{Laplace.} We apply a diagonal last-layer Laplace approximation to the trained FNO \citep{daxberger2021laplace,mackay1992practical,ritter2018scalable,magnaniKERH25}. The posterior variance, computed from the feature representation of each input, serves as the local uncertainty scale.
    \item \textbf{UQNO.} Following \citet{ma2024calibrated}, we train an auxiliary quantile FNO on a disjoint training subset to estimate the local $0.99$-quantile of the absolute residual. This quantile estimate is then normalized via the standard split-conformal calibration step.
    \item \textbf{Perturbation.} The proposed method (Section \ref{sec:method}). The local scale is derived from the spatially smoothed pointwise disagreement between a base FNO and a counterpart trained on Gaussian-perturbed labels, without requiring a separate uncertainty network.
    \item \textbf{Unscaled.} A baseline conformal procedure that sets $\sigma \equiv 1$ in \eqref{eq:normalized_score_template}, yielding constant-radius bands. The base operator $\widehat{\mathcal{G}}$ is trained identically to the perturbation method to isolate the effect of adaptive scaling.
\end{itemize}

For MC Dropout, Laplace approximation, Perturbation and Unscaled, we use $800$ trajectories for training, $200$ trajectories for calibration, and $200$ trajectories for testing. For UQNO, we partition the $800$ training trajectories into two disjoint subsets of size $400$: one trains the base predictive operator, and the other trains the auxiliary quantile network that estimates the local uncertainty scale.


After training, we pool the $200$ calibration and $200$ test trajectories and perform $1000$ independent random partitions into calibration and test sets of size $200$. All methods are evaluated on the identical sequence of splits. For a fixed miscoverage level $\alpha$, we report two aggregate metrics across the reshuffles: simultaneous coordinate coverage and average conformal radius. Simultaneous coverage measures the proportion of test trajectories for which the entire predicted spatiotemporal field falls within the conformal band:
\begin{equation}\label{eq:simul_metric}
\mathrm{Cov}_{\mathrm{sim}}
:=
\frac{1}{|\mathcal{I}_{\mathrm{test}}|}
\sum_{i\in\mathcal{I}_{\mathrm{test}}}
\mathbb{I}\bigl\{\mathbf{u}_i^{k}\in\mathcal{C}(\mathbf{a}_i)^{k} \text{ for all } k\in\{1,\dots,d_{\mathrm{out}}\}\bigr\}.
\end{equation}
This is the validity notion targeted by the max-type nonconformity score. Efficiency is quantified by the average conformal radius, defined as the mean half-width of the prediction band across all output coordinates:
\begin{equation}\label{eq:radius_metric}
\mathrm{Rad}(\mathbf{a}_i)
:=
\frac{1}{d_{\mathrm{out}}}
\sum_{k=1}^{d_{\mathrm{out}}}
\widehat{Q}_{1-\alpha}\,\sigma(\mathbf{a}_i)^{k}.
\end{equation}
The reported efficiency metric is the sample mean $\frac{1}{|\mathcal{I}_{\mathrm{test}}|}\sum_{i\in\mathcal{I}_{\mathrm{test}}}\mathrm{Rad}(\mathbf{a}_i)$; the full interval width equals twice this value. Standard errors are computed across the $1000$ reshuffles as $\mathrm{SE}=\mathrm{std}/\sqrt{1000}$ and reflect split-resampling variability conditional on the fixed trained models.

\subsection{Experiment Results}\label{sec:results}
Table~\ref{tab:uq} summarizes the main quantitative comparison. Across all reported significance levels, Perturbation achieves the smallest average conformal radius while maintaining empirical simultaneous coverage close to the nominal target. The deviations from the target level are on the order of the reported Monte Carlo standard errors. At $\alpha=0.04$, for example, the perturbation method has average radius $3.786\pm0.004$, compared with $4.192\pm0.002$ for Unscaled, $4.403\pm0.003$ for Laplace, $8.530\pm0.015$ for MC Dropout, and $16.784\pm0.091$ for UQNO. These results reflect the advantage of the perturbation construction under a fixed label budget: both the base and perturbed operators are trained on the full proper training set, whereas UQNO partitions the same data between a base predictor and an auxiliary uncertainty model.

\begin{table}[t]
    \centering
    \small
    \resizebox{\textwidth}{!}{%
    \begin{tabular}{lccccc}
        \toprule
        & $\alpha=0.02$ & $\alpha=0.04$ & $\alpha=0.06$ & $\alpha=0.08$ & $\alpha=0.10$ \\
        & target $98\%$ & target $96\%$ & target $94\%$ & target $92\%$ & target $90\%$ \\
        \midrule
        Unscaled cov. (\%) & 98.00 & 96.04 & 94.02 & 92.04 & 90.13 \\
        Unscaled rad. & $4.353\pm0.003$ & $4.192\pm0.002$ & $4.084\pm0.002$ & $3.986\pm0.002$ & $3.908\pm0.002$ \\
        \addlinespace
        MC Dropout cov. (\%) & 97.96 & 95.95 & 93.94 & 91.94 & 89.98 \\
        MC Dropout rad. & $10.289\pm0.029$ & $8.530\pm0.015$ & $7.671\pm0.012$ & $7.169\pm0.008$ & $6.875\pm0.005$ \\
        \addlinespace
        Laplace cov. (\%) & 97.95 & 95.97 & 94.01 & 92.07 & 90.09 \\
        Laplace rad. & $4.563\pm0.002$ & $4.403\pm0.003$ & $4.261\pm0.003$ & $4.143\pm0.003$ & $4.052\pm0.003$ \\
        \addlinespace
        UQNO cov. (\%) & 98.09 & 96.13 & 94.14 & 92.18 & 90.16 \\
        UQNO rad. & $27.210\pm0.154$ & $16.784\pm0.091$ & $12.605\pm0.060$ & $9.866\pm0.044$ & $8.030\pm0.037$ \\
        \addlinespace
        Perturbation cov. (\%) & 97.99 & 96.01 & 94.00 & 92.00 & 90.08 \\
        Perturbation rad. & $\mathbf{4.169\pm0.006}$ & $\mathbf{3.786\pm0.004}$ & $\mathbf{3.595\pm0.004}$ & $\mathbf{3.434\pm0.004}$ & $\mathbf{3.291\pm0.004}$ \\
        \bottomrule
    \end{tabular}}
    \caption{Empirical simultaneous-in-coordinate coverage and average conformal radius over $1000$ reshuffles of the calibration/test pool. All methods use the same total label budget. Values are mean $\pm$ standard error. The radius is the average half-width of the conformal band; the full interval width is twice this value.}
    \label{tab:uq}
\end{table}

Table~\ref{tab:relative_improvements} reports the corresponding relative radius reduction of the perturbation method compared with each baseline. The improvement is most pronounced against UQNO and MC Dropout, reflecting respectively the data-splitting penalty of training a separate uncertainty network and the high variance of dropout-based uncertainty estimates. The improvement over the unscaled and Laplace baselines is smaller but still systematic at the reported levels.

\begin{table}[t]

    \centering
    \small
    \begin{tabular}{lccc}
        \toprule
        Baseline & $\alpha=0.02$ & $\alpha=0.04$ & $\alpha=0.06$ \\
        \midrule
        vs. Unscaled & $4.2\%$ & $9.7\%$ & $12.0\%$ \\
        vs. Laplace & $8.6\%$ & $14.0\%$ & $15.6\%$ \\
        vs. MC Dropout & $59.5\%$ & $55.6\%$ & $53.1\%$ \\
        vs. UQNO & $84.7\%$ & $77.4\%$ & $71.5\%$ \\
        \bottomrule
    \end{tabular}
    \caption{Relative radius reduction of the perturbation method, computed as $(\mathrm{Rad}_{\mathrm{baseline}}-\mathrm{Rad}_{\mathrm{ours}})/\mathrm{Rad}_{\mathrm{baseline}}\times100$.}
    \label{tab:relative_improvements}
\end{table}

\subsection{Ablation: smoothing and floor stabilization}\label{sec:ablation_smoothing_floor}
The perturbation method applies two post-processing steps to the raw disagreement $d(\mathbf a)^k=|\widehat{\mathcal G}(\mathbf a)^k-\widehat{\mathcal G}_{\veps}(\mathbf a)^k|$: spatial smoothing with a $K\times K$ averaging window and floor stabilization $\sigma(\mathbf a)^k\ge\tau_0$. Both are intended to prevent isolated near-zero denominator values from inflating the max-type calibration score. Table~\ref{tab:smoothing_ablation} tests four configurations while keeping the trained models fixed. The ablation shows that spatial smoothing is the critical component. At $\alpha=0.04$, removing the floor while retaining the $15\times15$ smoothing window changes the average radius only from $3.786$ to $3.789$. In contrast, removing smoothing while retaining the floor increases the average radius to $83.616$. Removing both smoothing and the floor makes the method essentially unusable, producing astronomically large radii. Coverage remains close to the conformal target because validity does not depend on the quality of the scale, but efficiency depends strongly on stabilizing the denominator.

\begin{table}[!ht]
    \centering
    \footnotesize
    \begin{tabular}{lccc}
        \toprule
        Configuration & $\alpha=0.02$ & $\alpha=0.04$ & $\alpha=0.06$ \\
        \midrule
        Full ($K=15$, $\tau_0>0$) & $4.181\pm0.006$ & $3.786\pm0.004$ & $3.596\pm0.004$ \\
        No smoothing ($K=1$, $\tau_0>0$) & $87.573\pm0.058$ & $83.616\pm0.069$ & $80.430\pm0.050$ \\
        No floor ($K=15$, $\tau_0=0$) & $4.173\pm0.006$ & $3.789\pm0.004$ & $3.592\pm0.004$ \\
        Neither ($K=1$, $\tau_0=0$) & $2.1\times10^{7}$ & $2.0\times10^{5}$ & $1.4\times10^{5}$ \\
        \bottomrule
    \end{tabular}
    \caption{Ablation of spatial smoothing and floor stabilization. Values are average conformal radii, reported as mean $\pm$ standard error over $1000$ reshuffles when available. The conformal coverage remains near the nominal target in all cases, but the efficiency of the bands is dramatically affected by smoothing.}
    \label{tab:smoothing_ablation}
\end{table}

\subsection{Sensitivity to perturbation noise}\label{sec:noise_sensitivity}
The label-noise scale $\sigma_{\veps} = c \cdot \mathrm{std}(\mathbf{u}_{\mathrm{train}})$ is the primary hyperparameter governing the perturbation procedure. We evaluate the sensitivity of the conformal bands to $c \in \{0.01, 0.02, 0.05, 0.10, 0.20\}$ by independently retraining the perturbed FNO for each value, while keeping the base operator, spatial smoothing window, and stabilization floor fixed. 
\begin{table}[!ht]
    \centering
    \footnotesize
    \begin{tabular}{lcccccc}
        \toprule
        & \multicolumn{2}{c}{$\alpha=0.02$} & \multicolumn{2}{c}{$\alpha=0.04$} & \multicolumn{2}{c}{$\alpha=0.06$} \\
        \cmidrule(lr){2-3}\cmidrule(lr){4-5}\cmidrule(lr){6-7}
        $c$ & Cov. (\%) & Rad. & Cov. (\%) & Rad. & Cov. (\%) & Rad. \\
        \midrule
        $0.01$ & 98.05 & $3.995\pm0.004$ & 96.05 & $3.786\pm0.003$ & 94.07 & $3.670\pm0.003$ \\
        $0.02$ & 98.02 & $3.507\pm0.006$ & 96.08 & $3.248\pm0.003$ & 94.09 & $3.089\pm0.003$ \\
        $0.05$ & 98.07 & $3.734\pm0.005$ & 96.13 & $3.407\pm0.003$ & 94.12 & $3.278\pm0.002$ \\
        $0.10$ & 98.03 & $3.965\pm0.007$ & 96.08 & $3.509\pm0.004$ & 94.11 & $3.337\pm0.002$ \\
        $0.20$ & 98.00 & $3.492\pm0.004$ & 95.92 & $3.266\pm0.003$ & 93.92 & $3.112\pm0.002$ \\
        \bottomrule
    \end{tabular}
    \caption{Sensitivity to the perturbation-noise multiplier $c=\sigma_{\veps}/\mathrm{std}(\mathbf{u}_{\mathrm{train}})$. Values are empirical coverage and average conformal radius over $1000$ reshuffles.}
    \label{tab:noise_sensitivity}
\end{table}

Table~\ref{tab:noise_sensitivity} shows that the method is robust across this range of perturbation magnitudes. At $\alpha=0.04$, the average radius ranges from $3.248$ to $3.786$, a variation of about $16.6\%$, and coverage stays close to the nominal target. Small perturbations can produce a weak disagreement signal relative to training stochasticity, while large perturbations may begin to distort the perturbed model. In practice, the range $c\in[0.02,0.10]$ appears safe for this benchmark.

\subsection{Qualitative structure and diagnostics}\label{sec:vis_diagnostics}
To complement the aggregate metrics, we examine the spatial distribution of the conformal bands to verify that the perturbation scale captures physically structured uncertainty. Figure~\ref{fig:vis} illustrates this analysis for a representative test trajectory. The right panel displays the perturbation-derived radius $Q_{1-\alpha}\sigma(\mathbf{a})$. The uncertainty is not spatially uniform: it widens around vortex cores, shear layers, and complicated interactions, while remaining comparatively small in smoother background regions. 


\begin{figure}[!ht]
    \centering
    \IfFileExists{fig_vis_uq.png}
    {\includegraphics[width=0.5\textwidth]{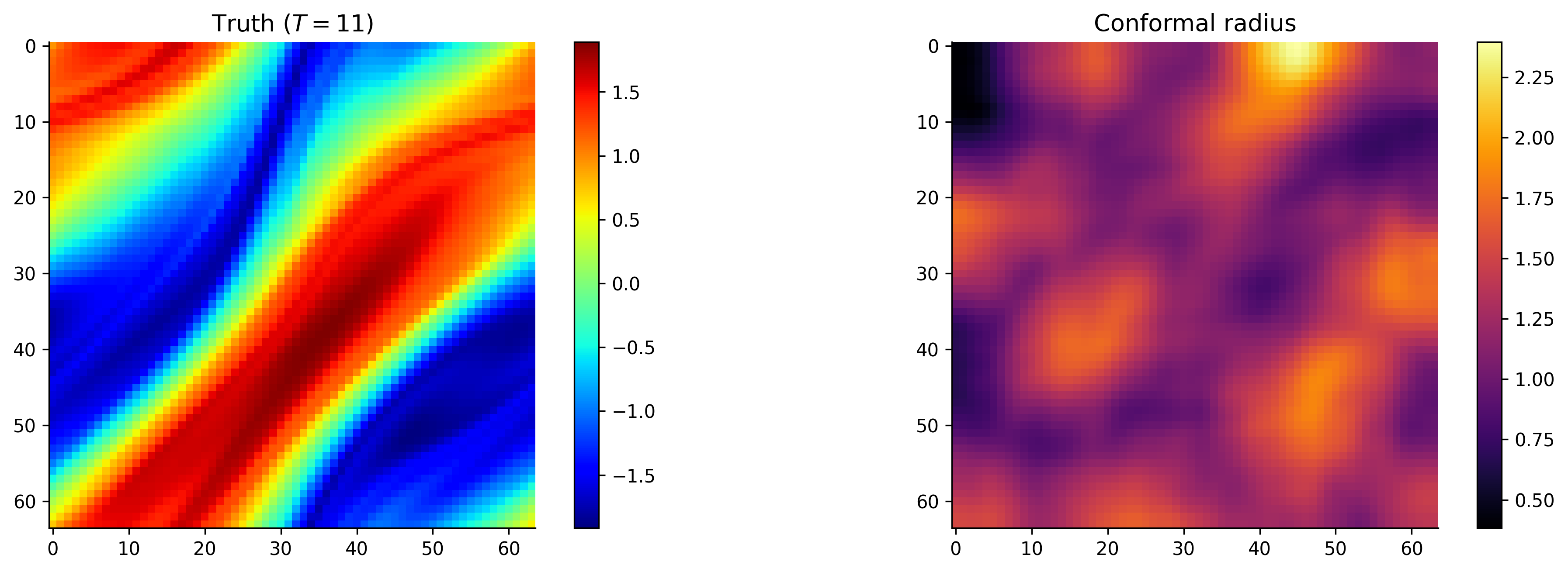}}{%
    \fbox{\parbox{0.92\textwidth}{\centering Placeholder for
    the qualitative perturbation-based uncertainty figure.
    Insert the final visualization here.}}}
    \caption{\textbf{Qualitative perturbation-based uncertainty
    map.} Left: ground-truth vorticity for a representative
    test trajectory. Right: conformal radius $Q_{1-\alpha}\sigma(a)$
    derived from the disagreement between the base and perturbed
    operators.}
    \label{fig:vis}
\end{figure}

\begin{figure}[!ht]
    \centering
    \IfFileExists{fno_result_64x64.jpg}{\includegraphics[width=0.9\textwidth]{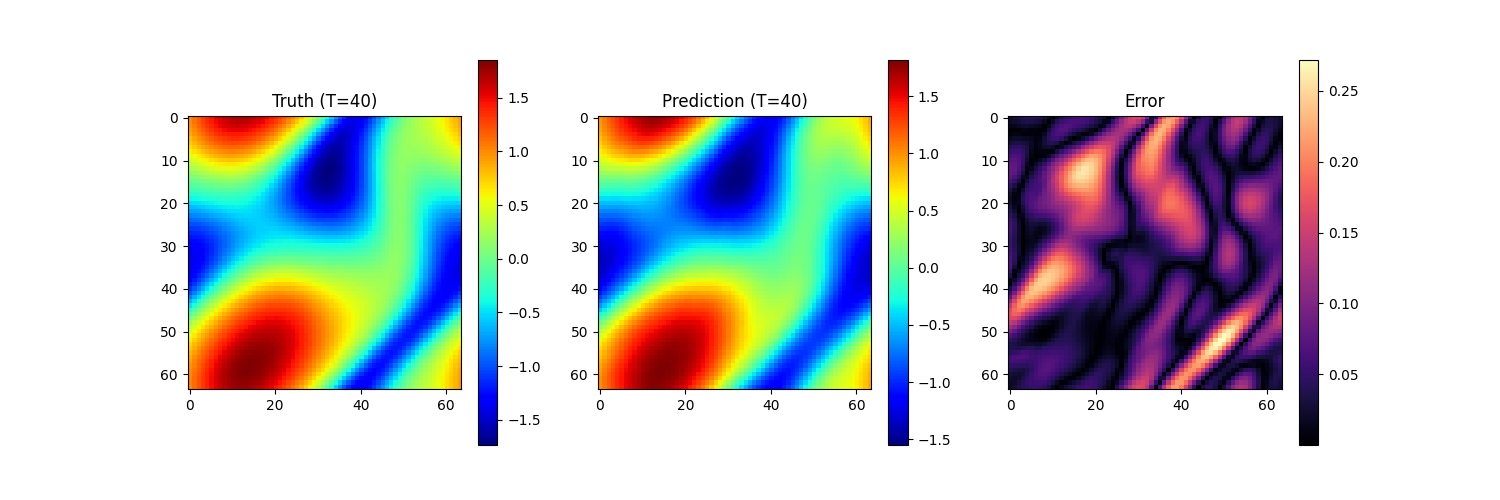}}{%
    \fbox{\parbox{0.92\textwidth}{\centering {Placeholder for the base-operator performance figure. Insert the final FNO visualization here.}}}}
    \caption{\textbf{Base operator performance check.} Left: ground-truth vorticity field for a representative test sample. Center: FNO prediction. Right: absolute error.}
    \label{fig:base_model}
\end{figure}

Figure~\ref{fig:base_model} illustrates the predictive accuracy of the base FNO. The absolute residual map indicates that prediction errors are localized, whereas the dominant flow structures are accurately recovered. This baseline accuracy ensures that the conformal bands quantify spatially varying predictive uncertainty, rather than compensating for large systematic errors in the surrogate.

\FloatBarrier

\section{Conclusion}\label{sec:conclusion}

We proposed a perturbation-based conformal prediction framework for uncertainty quantification in neural operator learning. By comparing a base predictor with a counterpart trained on perturbed labels, we obtain a spatially adaptive uncertainty scale that can be conformalized to yield finite-sample simultaneous coverage on the output grid. On the 2D Navier--Stokes benchmark, this construction is particularly effective in the data-scarce regime, where it outperforms MC Dropout, Laplace approximation, and UQNO under matched total label budgets.

More broadly, the perturbation-based scale offers a general mechanism for constructing adaptive nonconformity scores in regression problems. The construction requires only a base predictor, a second predictor trained on perturbed targets, and a pointwise disagreement measure between the two. This procedure does not rely on the PDE structure, the function-space formulation, or any architecture-specific feature; it applies whenever one can train two regression models on nearly identical data and wishes to calibrate prediction bands that widen where the learned response is sensitive to training noise. The conformalization step remains standard split conformal prediction, so the marginal coverage guarantee holds under the usual exchangeability assumption. By decoupling uncertainty calibration from auxiliary modeling, the perturbation-based score provides a data efficient alternative for conformalized regression across diverse output geometries. 




\section*{Acknowledgments}
WW would like to thank the Simons Foundation for their support through the Travel Support for Mathematicians program (No.\ 0007730). 


\end{document}